\documentclass[letterpaper]{article} 
\usepackage{aaai25} 
\usepackage{times}  
\usepackage{helvet} 
\usepackage{courier} 
\usepackage[hyphens]{url}
\usepackage{graphicx} 
\usepackage{natbib}  
\usepackage{caption} 
\usepackage{algorithm}
\usepackage{algorithmic}
\usepackage{amsmath}

\usepackage{subcaption}
\usepackage{xcolor}
\usepackage{booktabs}
\usepackage{amssymb}
\usepackage{amsthm}
\newtheorem{theorem}{Theorem}

\theoremstyle{definition}
\newtheorem{definition}{Definition}

\usepackage{newfloat}
\usepackage{listings}
\DeclareCaptionStyle{ruled}{labelfont=normalfont,labelsep=colon,strut=off} 
\floatstyle{ruled}
\newfloat{listing}{tb}{lst}{}
\floatname{listing}{Listing}

\title{Clustering by Mining Density Distributions and Splitting Manifold Structure}
\author{
    Zhichang Xu\textsuperscript{\rm 1}, 
    Zhiguo Long\textsuperscript{\rm 1}\thanks{Corresponding authors.}, Hua Meng\textsuperscript{\rm 2}\footnotemark[1]
}
\affiliations{
    \textsuperscript{\rm 1}School of Computing and Artificial Intelligence, Southwest Jiaotong University, Chengdu 611756, China\\
    \textsuperscript{\rm 2}School of Mathematics, Southwest Jiaotong University, Chengdu 611756, China\\
    2023201798@my.swjtu.edu.cn, \{zhiguolong, menghua\}@swjtu.edu.cn
}

\usepackage{bibentry}

\begin{document}

\maketitle

\begin{abstract}
Spectral clustering requires the time-consuming decomposition of the Laplacian matrix of the similarity graph, thus limiting its applicability to large datasets. To improve the efficiency of spectral clustering, a top-down approach was recently proposed, which first divides the data into several micro-clusters (granular-balls), then splits these micro-clusters when they are not ``compact'',
and finally uses these micro-clusters as nodes to construct a similarity graph for more efficient spectral clustering. 
However, this top-down approach is challenging to adapt to unevenly distributed or structurally complex data.
This is because constructing micro-clusters as a rough ball struggles to capture the shape and structure of data in a local range, 
and the simplistic splitting rule that solely targets ``compactness'' is susceptible to noise and variations in data density and leads to micro-clusters with varying shapes, making it challenging to accurately measure the similarity between them.
To resolve these issues and improve spectral clustering, this paper first proposes to start from local structures to obtain micro-clusters, such that the complex structural information inside local neighborhoods is well captured by them. 
Moreover, by noting that Euclidean distance is more suitable for convex sets, this paper further proposes a data splitting rule that couples local density and data manifold structures, so that the similarities of the obtained micro-clusters can be easily characterized. 
A novel similarity measure between micro-clusters is then proposed for the final spectral clustering.
A series of experiments based on synthetic and real-world datasets demonstrate that the proposed method has better adaptability to structurally complex data than granular-ball based methods.
\end{abstract}

\section{Introduction}
Clustering is an unsupervised learning method aiming to reveal the intrinsic distribution characteristics of data by dividing the dataset into several non-overlapping clusters. 
It has wide applications in various fields such as computer vision~\cite{DBLP:conf/iccv/TronZED17}, language processing~\cite{DBLP:conf/emnlp/0001WS23}, and bioinformatics~\cite{DBLP:journals/bioinformatics/ChengM22}.

Spectral lustering is a representative graph partition clustering algorithm, 
the core idea of which is to (approximately) minimize a cut loss of partitioning a similarity graph of data by removing some edges.
It has attracted wide attention
due to its exceptionally good performance in handling non-convex shaped
clusters~\cite{DBLP:journals/sac/Luxburg07,DBLP:journals/pami/ChenSBLC11,DBLP:journals/tcyb/HeRGZ19}.

\begin{figure}[t]
\centering
\includegraphics[width=0.45\columnwidth]{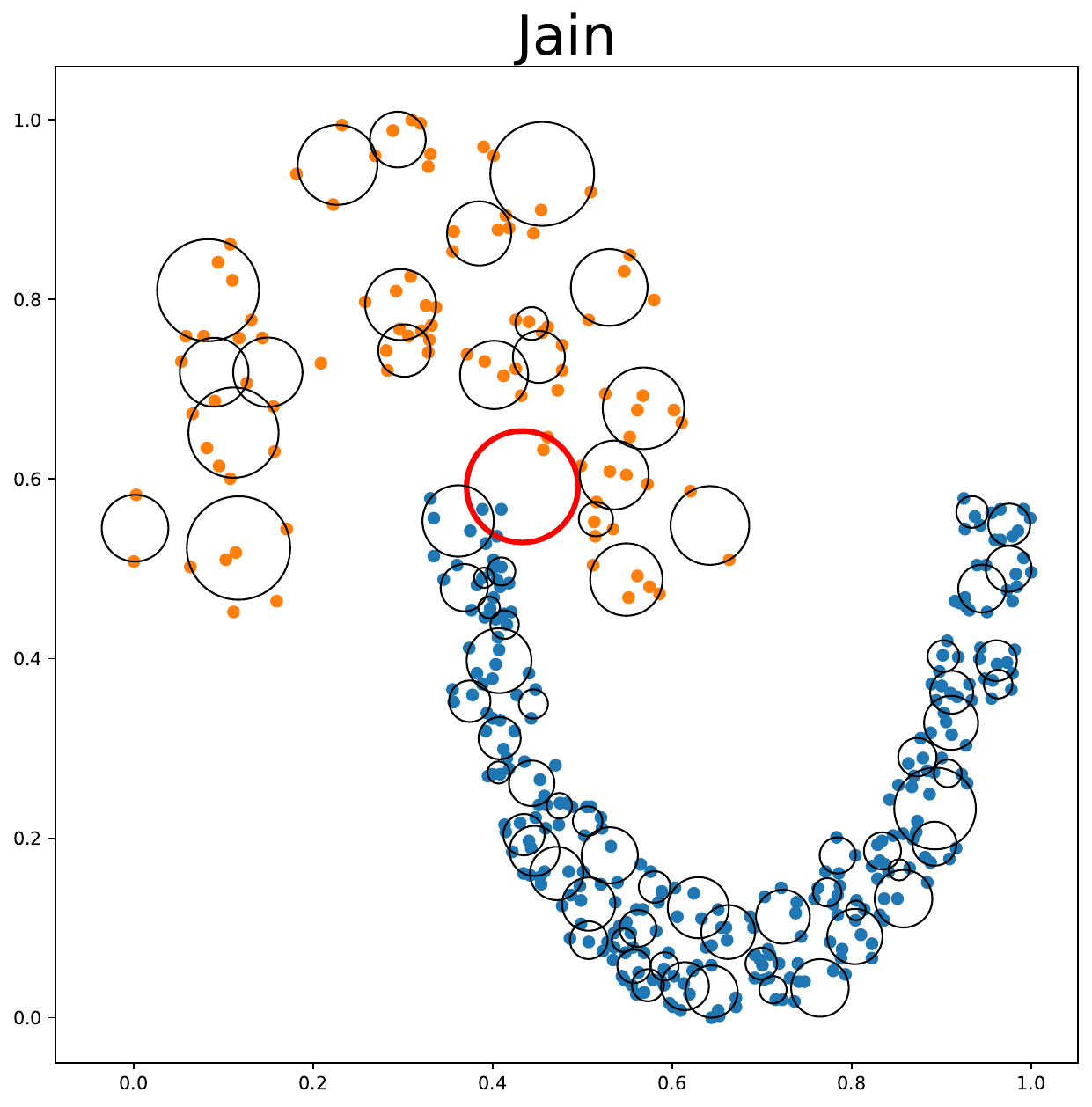} 
\includegraphics[width=0.45\columnwidth]{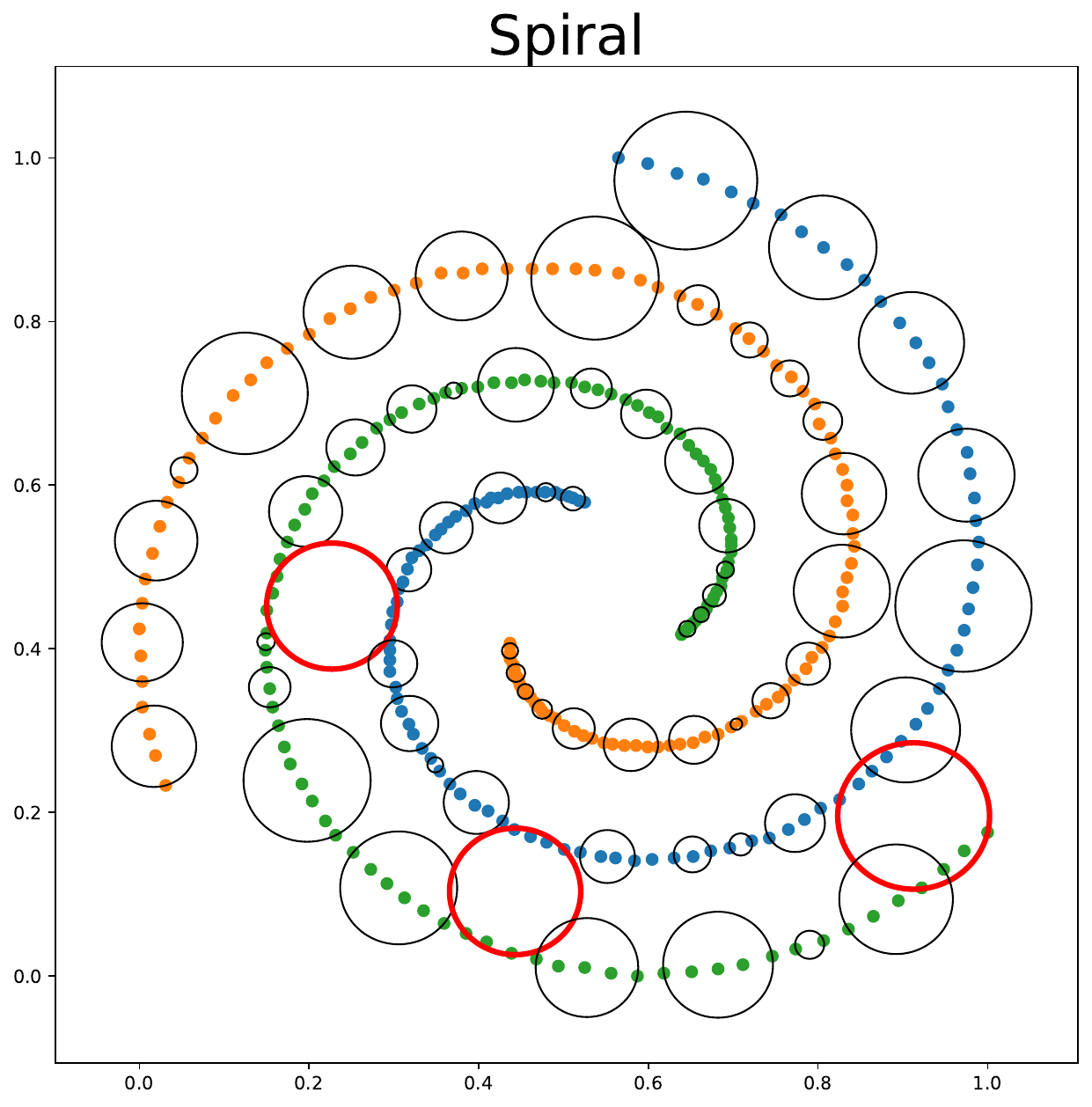} 
\caption{Illustration of errors (red circles) in granular-balls.}
\label{fig1}
\end{figure}

Spectral clustering involves the spectral decomposition of the Laplacian matrix of a similarity graph, 
which has a prohibitively high complexity of $\mathcal{O}(n^{3})$ ($n$ represents the number of nodes in the graph) for large datasets.
To improve the scalability, 
researchers have considered performing approximate spectral decomposition~\cite{DBLP:conf/icml/VladymyrovC16}, representing similarity between data approximately with anchors~\cite{DBLP:journals/tcyb/CaiC15,DBLP:journals/tkde/HuangWWLK20}, and constructing more sparse similarity graphs~\cite{DBLP:conf/kdd/WuCYXXA18}. 
These methods have different advantages and disadvantages, e.g., approximate decomposition needs to balance between efficiency and accuracy, and anchor-based ones are very sensitive to anchor numbers and positions. There is also research trying to fuse these directions~\cite{DBLP:journals/pr/YangDCCYGH23}.

Recently, \citet{DBLP:journals/tkde/XieKXWG23} proposed a new method called GBSC 
that progressively splits the data into micro-clusters (\emph{granular-balls}) in a top-down manner, 
to obtain a coarse-grained representation of the original data, 
and then performs spectral clustering on balls representing several similar data points to reduce the size of the similarity graph.
As each point is represented by a granular-ball, the overall structural
information of data would be better captured than using the sampled anchors
that is subject to missing the selection of some important anchors.
This makes the granular-ball based approach very promising.

However, as it heavily relies on the quality of balls and the accuracy of the similarity measure between the balls,
GBSC has two significant deficiencies for data of complex structures.
Firstly, the splitting rule could produce low quality balls for complex data, because the rule only targets more ``compact'' balls and is based on a global view of data and a top-down manner.
For example, in Fig.~\ref{fig1}, when generating granular-balls from two clusters with significant density differences or manifold structure, some boundary points from different clusters are incorrectly grouped into the same granular-ball.
Secondly, a granular-ball might contain data points distributed on a non-convex shape when the dataset is complex, and thus the Euclidean-based similarity between balls might no longer be appropriate.

To resolve these issues and accelerate spectral clustering, this paper proposes another approach to represent data in a more coarse granularity. It first tries to capture more detailed local structural information of data by constructing micro-clusters from an estimation of data densities.
Then it follows a more sophisticated splitting rule that also considers the convexity of data (called \emph{manifold curvature}) to improve the quality of micro-clusters. The more convex distribution of data in micro-clusters also makes the design of similarity measure more straightforward and thus the final spectral clustering more effective.

The contributions of this paper are as follows:
\begin{itemize}
\item We propose a coarse-grained data representation scheme that combines local density estimation and convex splitting of local manifolds that is exploited to improve spectral clustering. 

\item Compared to granular-balls, the complex structures of data are better captured by extracting local density features to form a coarse-grained representation of data.
\item Manifold curvature is introduced to split micro-clusters into more convex ones, which results in easier characterization of the intrinsic similarities between micro-clusters.

\end{itemize}

The rest of the paper is organized as follows. We first introduce the related work and review how granular-balls are generated, and then discuss the motivation and the framework of our algorithm, followed by experimental evaluations of the algorithm.

\section{Related Work}
The improvement and acceleration of spectral clustering has always been a focus in the field.
Earlier, researchers explored numerical computation methods to accelerate spectral decomposition, e.g., the Nystr{\"o}m method~\cite{DBLP:journals/pami/FowlkesBCM04} efficiently approximates the spectral decomposition of a large Laplacian matrix by sampling data points. Later research~\cite{DBLP:conf/icml/VladymyrovC16,DBLP:conf/nips/Macgregor23} improved the sampling strategy and the representation ability to increase stability and reduce approximation error. 
However, these algorithms face the issue of balancing efficiency and approximation accuracy, and sampling-based schemes also have the problem of instability.

Another direction is to use anchors to characterize data, where the similarity between points are approximated using the similarity between points and anchors.
Since the similarity graph between points and anchors is a bipartite graph, 
the spectral decomposition can be more efficient by working on a smaller matrix.
For example, \citet{DBLP:journals/tcyb/CaiC15} performed K-means to obtain anchor points and constructed a sparsified bipartite graph between data points and these anchors by keeping only the connections of several nearest anchors for each point.
\citet{DBLP:journals/tkde/HuangWWLK20} further improved computational efficiency by first applying K-means on randomly sampled points to obtain anchors, then constructing a bipartite graph via fast approximate nearest neighbors, and finally performing transfer-cut on the bipartite graph for efficient spectral decomposition. There are also works~\cite{DBLP:journals/pr/GaoCNYW24,DBLP:journals/pami/NieXYL24} on combining anchors with the optimization process to increase efficiency. 

Sparse similarity graphs can also be used to accelerate spectral clustering.
For example, the SCRB method~\cite{DBLP:conf/kdd/WuCYXXA18} used random binning features to generate inner products that approximate the similarity matrix of data, and employed singular value decomposition of large sparse matrices to improve the efficiency of spectral decomposition. 
The RESKM \cite{DBLP:journals/pr/YangDCCYGH23} framework tried to ensemble multiple strategies for more efficient spectral clustering.

Using micro-clusters to represent a group of data points was also a promising direction 
to reduce the size of data used for spectral clustering.
The KASP method~\cite{DBLP:conf/kdd/YanHJ09} used K-means to group data points into micro-clusters and performed spectral clustering on the centers of these micro-clusters to reduce the size of similarity graph.
The GBSC method~\cite{DBLP:journals/tkde/XieKXWG23} is a new and more sophisticated way to group data points into micro-clusters,  
by splitting the micro-clusters in a top-down manner to obtain a coarse-grained representation of the data, aiming for more compact micro-clusters. 
This method is combined with spectral clustering to achieve final clustering, while it can be combined with other clustering methods~\cite{cheng2023fast} as well, demonstrating promising prospects. 
However, this top-down manner for a coarse-grained representation can result in micro-clusters of low quality which can distort the final clustering result. 
Therefore, this paper proposes a new micro-cluster construction strategy based on local structures, and also optimizes the splitting strategy to make the micro-clusters more convex and with high purity, to align with the needs of spectral clustering.

\begin{figure*}[t]
\centering
\includegraphics[width=0.9\textwidth]{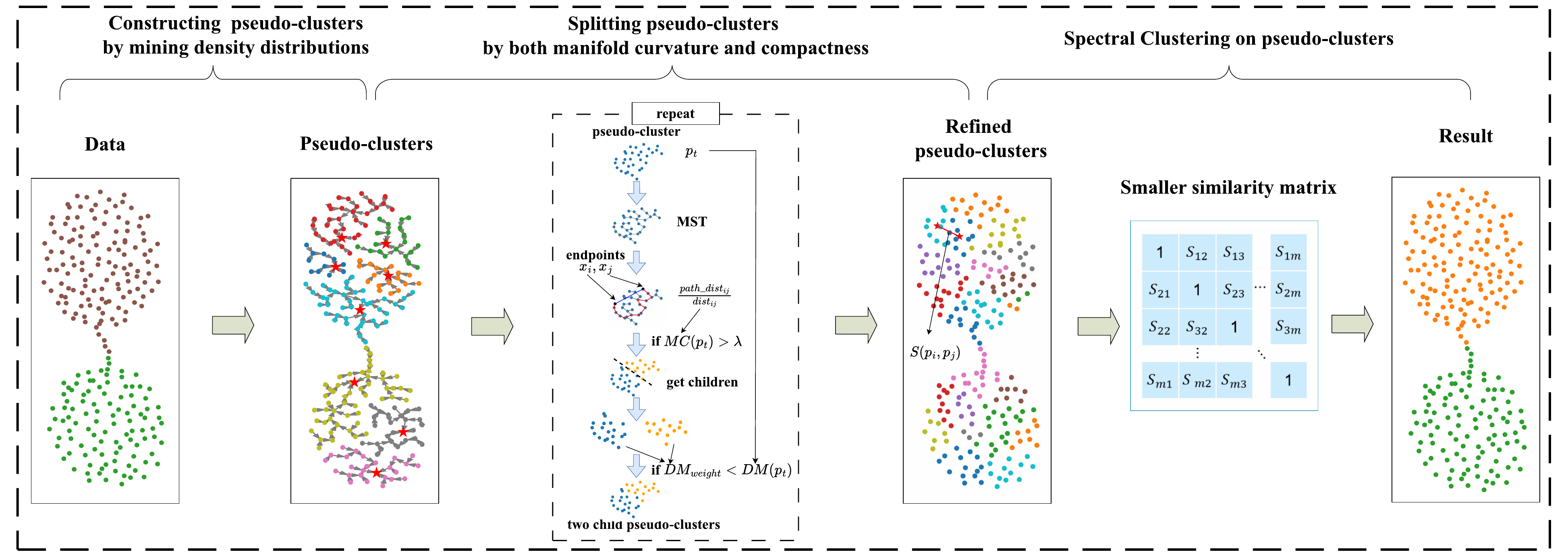} 
\caption{The framework of the proposed MDMSC algorithm.}
\label{fig2}
\end{figure*}

\section{Granular-Ball Generation}

Given a dataset $X=\{x_1,...,x_n\}$ $(x_i\in \mathbb{R}^{d})$, 
the target of GBSC~\cite{DBLP:journals/tkde/XieKXWG23} is to use a set of \emph{granular-balls} as micro-clusters to cover the dataset, such that each data point belongs to a single granular-ball. 
A granular-ball is just an $d$ dimensional ball while the radius of the ball can vary from each other.

Specifically, suppose $GB_j$ is a granular-ball covering the $m$ data points $\{x_{j_1},\ldots,x_{j_{m}}\}\subseteq X$, then 
the center $c_j$ and the radius of $GB_j$ 
are determined as $c_j=\frac{1}{m}\sum_{s=1}^{m}x_{j_s}$ and $r_j=\max_s(\begin{Vmatrix}x_{j_s}-c_j\end{Vmatrix})$, where $||\cdot||$ denotes the $\ell_2$ norm.
Note that if the radius of the granular-ball is large then the ``granularity'' is coarse and the clustering on the balls would be fast while much structural information would be lost;
otherwise, the ``granularity'' is fine but clustering would be slow.
Therefore, generating granular-balls needs to balance granularity and the \emph{quality} of the balls.

In GBSC, the quality of a granular-ball $GB_j$ is defined as 
\begin{equation}\label{eq:dm}
DM_j=\frac{1}{m}\sum_{s=1}^{m}\|x_{j_s}-c_j\|. 
\end{equation}
$DM_j$ actually measures the ``compactness'' of the points within $GB_j$, where a smaller value of $DM_j$ means that the distance between the points is mostly small.

The generation of granular-balls in a top-down manner works as follows.
First, a granular-ball $GB_A$ covering the entire dataset is generated.
Then, two farthest points $p_1$ and $p_2$ are selected, and the points in $GB_A$ are assigned to $p_1$ if they are closer to $p_1$ than to $p_2$.
Two children balls $GB_{A_1}$ and $GB_{A_2}$ are then generated using the two subsets of points.

The core splitting strategy in GBSC is that
if the \emph{weighted quality} of the children balls is higher than that of the parent ball.
Suppose $GB_{j_1}$ and $GB_{j_2}$ are the two children balls of $GB_j$, covering $m_1$ and $m_2$ points, respectively. 
Then the \emph{weighted quality} of the children balls is 
\begin{equation}\label{eq:dmw}
DM_{weight}=\frac{m_1}{m}DM_{j_1}+\frac{m_2}{m}DM_{j_2}. 
\end{equation}

For corner cases, GBSC also provides other splitting strategies including the restrictions on the number of points and the radius of a ball (see~\cite{DBLP:journals/tkde/XieKXWG23}).
The above splitting process is repeated to generate final granular-balls until
no more split can happen.

\section{Algorithm}
The splitting scheme of GBSC is performed in a top-down manner, and it does not consider local information, which can lead to incorrect splitting of local structures and thus affect the clustering performance.

To address these issues, we propose an improved algorithm MDMSC that initially partitions the dataset into multiple pseudo-clusters, instead of granular-balls, based on the density distribution of the data points,
and then further split the pseudo-clusters based on structural characteristics. 
This approach aims to better capture the local features of the dataset and improve the clustering performance on complex datasets.

MDMSC consists of three stages: constructing pseudo-clusters as micro-clusters, splitting pseudo-clusters, and performing spectral clustering based on the similarity of the final pseudo-clusters. The algorithm framework is illustrated in Fig. \ref{fig2}.

\subsection{Constructing Pseudo-Clusters}
For each point $x_i$ in a dataset $X\in P^{n\times d}$, let $N_k(x_i)$ be the set of the $k$ nearest neighbors of $x_i$ (excluding $x_i$ itself). 

Before defining pseudo-clusters, we need to first define \emph{density} and \emph{leader}.

\begin{definition} 
    The \emph{density} $\rho_k(x_i)$ of a point $x_i$, is defined as
\begin{equation}
    \rho_k(x_i)=\sum_{j=1}^k\exp(-dist_{ij}^2),
\end{equation}
where $dist_{ij}$ refers to the Euclidean distance between point $x_i$ and its neighbor $x_j$.  The density is measured by the sum of Gaussian kernels of the Euclidean distances, reflecting the compactness of the local structure around the point.
\end{definition}

The \emph{leader} of each point $x_i$, denoted as $\mathsf{leader}(x_i)$, is defined as the nearest higher-density point of $x_i$.
\begin{definition}
    The leader of a point $x_i$ is 
\begin{equation}
    \mathsf{leader}(x_i)=\left\{
        \begin{array}{ll}
            \mathop{\text{argmin}}\limits_{x_j \in \mathcal{H}(x_i)}{dist_{ij}}& \text{if $\mathcal{H}(x_i)\neq\emptyset$}\\
            \text{None}& \text{otherwise}
        \end{array}\right.
\end{equation}
where the set $\mathcal{H}(x_i)=\{x_j \mid x_j\in N_k(x_i), \rho_k(x_j)>\rho_k(x_i)\}$.
Points without a leader are called \emph{core points}, and the set of core points is denoted as $\mathsf{core}=\{x_i \mid \mathsf{leader}(x_i)=\text{None}\}$.
\end{definition}

\begin{definition}[Pseudo-cluster]
    Let $G=(X,E)$ be a directed graph, where $(x_i,x_j) \in E$ if $x_j$ is the leader of $x_i$. A \emph{pseudo-cluster} is a connected component of $G$.
\end{definition}
Pseudo-clusters are actually a local tree structure, where its root is a core point. 
By considering density distributions, a local tree structure has high purity and can thus better reflect local structures of data than granular-balls.

Intuitively, by connecting each point $x_i$ to its leader $\mathsf{leader}(x_i)$, multiple disjoint pseudo-clusters 
are formed, with each pseudo-cluster being defined by the unique core point within the pseudo-cluster. Algorithm~\ref{alg:algorithm} shows the steps of constructing pseudo-clusters. 

Fig.~\ref{fig3} illustrates the pseudo-clusters of Jain with $k=10$ and Spiral with $k=4$. It can be seen that pseudo-clusters better reflect the local characteristics of the dataset, and the problem of a granular-ball in Fig.~\ref{fig1} containing points from different clusters has been rectified. 
\begin{figure}[tb]
\centering
\includegraphics[width=0.45\columnwidth]{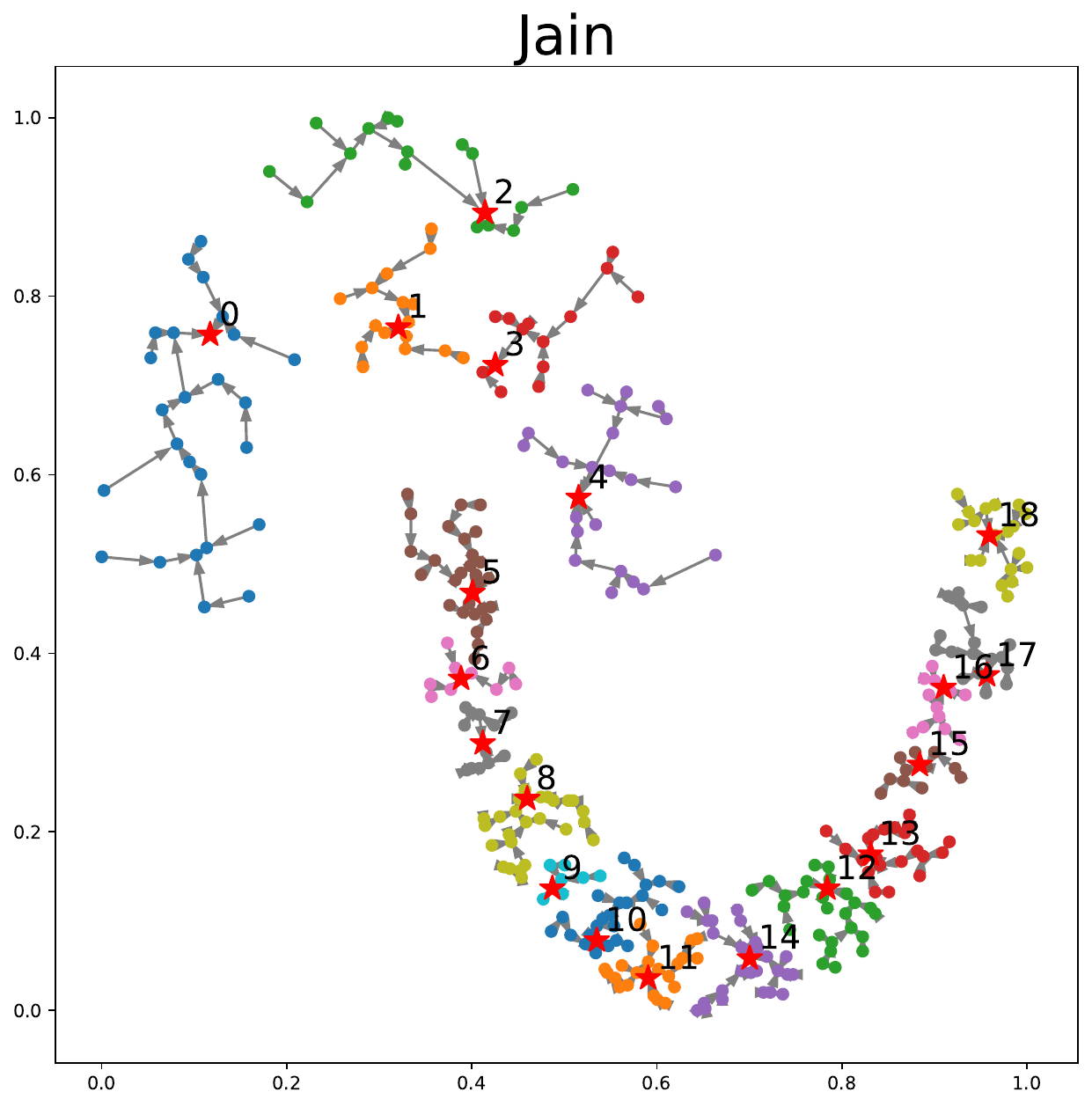} 
\includegraphics[width=0.45\columnwidth]{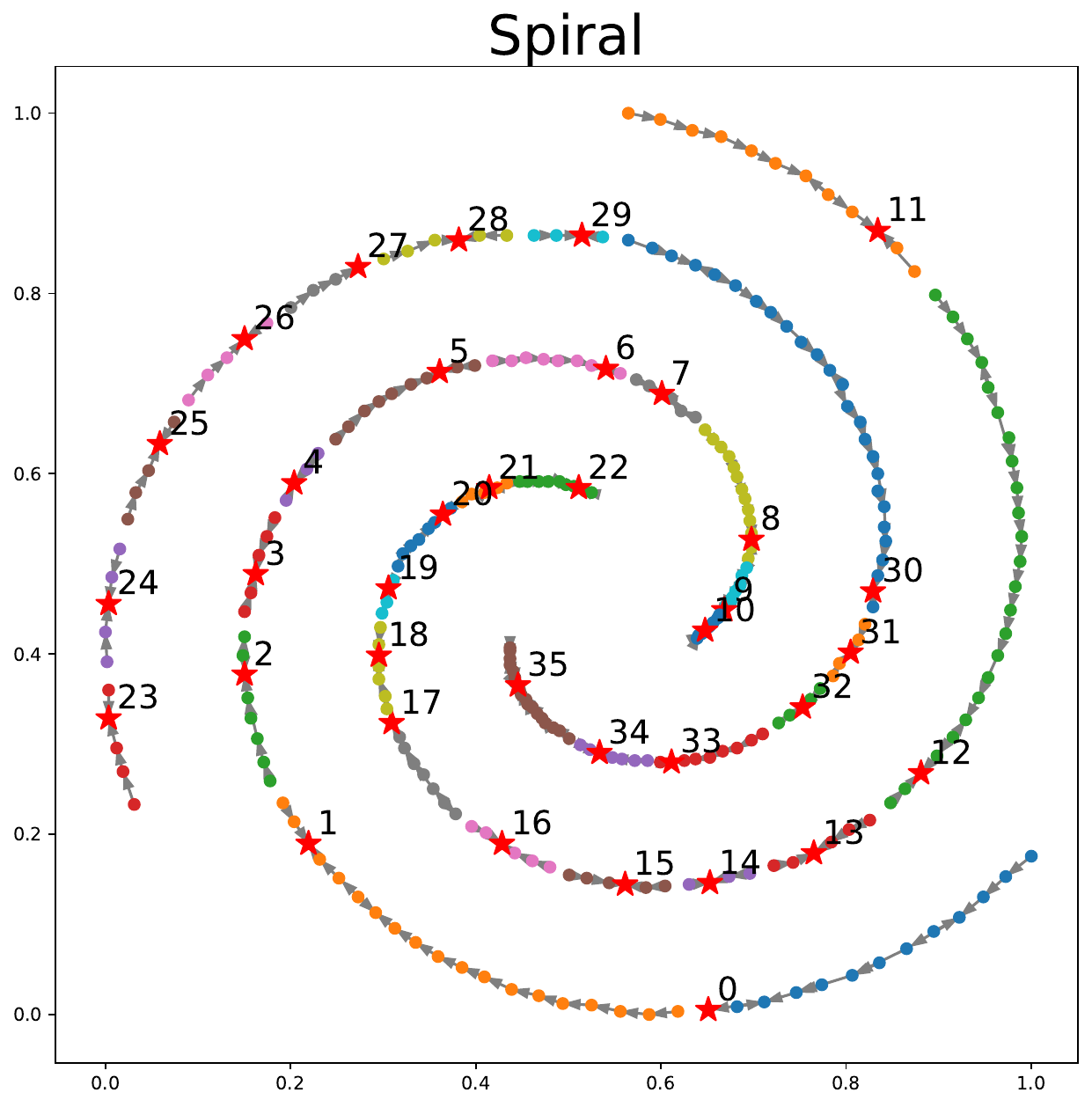} 
\caption{Illustration of pseudo-clusters and their complex structures.}
\label{fig3}
\end{figure}

\begin{algorithm}[tb]  
\caption{Constructing pseudo-clusters}  
\label{alg:algorithm}  
\begin{algorithmic}[1] 
\REQUIRE Dataset X, and the number of nearest neighbors $k$.
\ENSURE Pseudo-clusters $\mathsf{pseudo}\_\mathsf{clusters}$.
\FOR{each $x_i \in X$}   
\STATE Calculate $N_k(x_i)$;  
\STATE Calculate $\rho_k(x_i)$;
\ENDFOR
\STATE $\mathsf{core} \gets \emptyset$;  
\FOR{each $x_i \in X$}  
\STATE Calculate $\mathsf{leader}(x_i)$;  
\IF {$\mathsf{leader}(x_i)$ = None}  
\STATE \quad $\mathsf{core} \gets \mathsf{core} \cup \{x_i\}$;  
\ENDIF  
\STATE Connect $x_i$ and $\mathsf{leader}(x_i)$;
\ENDFOR  
\STATE Identify $\mathsf{pseudo}\_\mathsf{clusters}$ as connected components;
\STATE \textbf{return} $\mathsf{pseudo}\_\mathsf{clusters}$
  
\end{algorithmic}  
\end{algorithm}

\subsection{Splitting Pseudo-Clusters}
Although the pseudo-clusters reflect the basic distribution of the data, these pseudo-clusters may be too coarse for the entire dataset. 
Additionally, 
pseudo-clusters have various shapes and structures, many of which are non-convex, which makes measuring the similarity between pseudo-clusters difficult.

On the left side of Fig. \ref{fig3}, the pseudo-clusters of numbers 1, 2, 3, and 4 exhibit
non-convex structures and complex shapes, and the distance between these pseudo-clusters and others 
can not be easily measured with Euclidean based distances.
For example, for pseudo-cluster 2, although its core point is close to pseudo-cluster 1, the points in the left part of it are far away. The situation is similar on the other side.

To address this issue, we need to split the pseudo-clusters into simpler convex structures.
We introduce the measurement of \emph{manifold curvature} to better determine 
if a pseudo-cluster is too ``curved''. 
\begin{definition}[Manifold curvature] 
    Suppose $p_t$ is a pseudo-cluster, and $T(p_t)$ is a minimum spanning tree (MST) of the complete weighted graph for the points in $p_t$, where the weights of the edges are the Euclidean distances between the points.
    The \emph{manifold curvature} of a pseudo-cluster $p_t$ is 
\begin{equation}
    MC(p_t)=\frac{path\_dist_{ij}}{dist_{ij}} 
\ (x_i,x_j=\mathsf{endpoints}(p_t)),
\end{equation}
where $path\_dist$ is the shortest path distance between points on $T(p_t)$, and $\mathsf{endpoints}(p_t)=\operatorname{argmax}_{x_i,x_j\in p_t} path\_dist_{ij}$.
\end{definition}

 Intuitively, if the geodesic distance (the shortest distance on a manifold) equals the Euclidean distance, then the set is convex, and non-convex otherwise. Here, we use the shortest path distance in MST to approximate the geodesic distance, and the ratio of them to measure curvature.

Thus, we define the manifold curvature as the ratio of the shortest path distance to the Euclidean distance between the endpoints. 
The larger the ratio is, the more ``curved'' the pseudo-cluster will be.
When the ratio approaches 1, the data will be nearly convex.

By splitting pseudo-clusters under the measurement of manifold curvature, 
we can ensure that the pseudo-clusters have stronger convexity, 
making the consideration of pseudo-cluster compactness more reasonable and also beneficial for capturing the intrinsic similarities.

To determine whether a pseudo-cluster $p_t$ should be split, 
we consider both the manifold curvature and the compactness of the original pseudo-cluster and its child pseudo-clusters according to Eq.~\ref{eq:dm} and Eq.~\ref{eq:dmw}.
In particular, if $MC(p_t)\ge \lambda$ and $DM_{weight} < DM(p_t)$, then $p_t$ will be split. Because geodesic distance is estimated by graph distance, the two distances would be equal only with sufficient sampling; otherwise, graph distance is larger than the true geodesic distance. Thus, we set a threshold of 1.5 for the $\lambda$. In order to avoid too small pseudo-clusters, we also require $p_t$ should contain a minimum number $\beta$ of points, as in~\cite{DBLP:journals/tkde/XieKXWG23}.

When a pseudo-cluster needs to be split, 
its two endpoints $\mathsf{endpoints}(p_t)$ will be used to construct
the new child pseudo-clusters, where the other points in $p_t$ are assigned to the new pseudo-clusters based on proximity to the endpoints. 
The process will be repeated until none of the pseudo-clusters can be split.

\begin{algorithm}[tb]
\caption{MDMSC}
\label{alg:algorithm2}
\begin{algorithmic}[1]
\REQUIRE Dataset $X$, the number of nearest neighbors $k$, the manifold curvature threshold $\lambda$, and the minimum size of pseudo-cluster $\beta$.
\ENSURE Final clustering labels.
\STATE Get $\mathsf{pseudo}\_\mathsf{clusters}$ by Algorithm~\ref{alg:algorithm};
\FOR{each $p_t \in \mathsf{pseudo}\_\mathsf{clusters}$}
    \STATE Generate a complete undirected graph $G(p_t)$;
    \STATE Generate an MST $T(p_t)$ from $G(p_t)$;
    \STATE Calculate $DM_{weight}$, $DM(p_t)$ and $MC(p_t)$;
\ENDFOR
\WHILE{$\exists p_t$ s.t. $MC(p_t) >\lambda$ and $DM_{weight} < DM(p_t)$ and $|p_t|>\beta$}
    \STATE Split $p_t$ into $p_{t_1}$ and $p_{t_2}$ and add them to $\mathsf{pseudo}\_\mathsf{clusters}$; 
    \STATE Remove $p_t$ from $\mathsf{pseudo}\_\mathsf{clusters}$; 
\ENDWHILE
\STATE Calculate similarity matrix $S$;
\STATE Perform spectral clustering on the $S$ to obtain final clustering results.
\end{algorithmic}
\end{algorithm}

\subsection{Clustering Pseudo-Clusters}
After obtaining the final pseudo-clusters, 
the next step is to establish similarity between pseudo-clusters and ultimately perform clustering through graph partitioning
Here, we employ spectral clustering on pseudo-clusters.

Since the pseudo-clusters are approximately convex now, it is straightforward to use Euclidean distance to evaluate their similarity, as follows.
\begin{definition}[Similarity] 
    The similarity of two pseudo-clusters $p_i$ and $p_j$ is
\begin{equation}
    S(p_i,p_j)=\frac{|\mathsf{SNN}(p_i,p_j)|}{1+c\_dist(p_i,p_j)},
\end{equation}
where $\mathsf{SNN}(p_i,p_j) = (\cup_{x\in p_i} N_k(x)) \cap (\cup_{x\in p_j} N_k(x)$) is the shared nearest neighbors of $p_i$ and $p_j$,
and $c\_dist(p_i,p_j)$ is the Euclidean distance between the centroids of $p_i$ and $p_j$. A centroid of $p_i$ is $\frac{1}{m_i} \sum_{x\in p_i} x$, which is not necessarily a core point.
\end{definition}

Subsequently, we perform spectral clustering on the similarity matrix of the pseudo-clusters. 
Points within the same pseudo-cluster will be assigned the same cluster label.
The full steps of the proposed algorithm MDMSC is given in Algorithm~\ref{alg:algorithm2}.
\subsection{Time Complexity}

Suppose the dataset $X$ has $n$ samples with the dimensionality of $d$, the final number of pseudo-clusters is $m$, and the number of nearest neighbors is $k$. In Algorithm 1, the time complexity for finding k-nearest neighbors using the KD-tree method is $\mathcal{O}((d+k)n\log n)$, the time complexity for calculating density is $\mathcal{O}(kn)$, and the time complexity for finding the leader points is $\mathcal{O}(kn)$.

In Algorithm 2, suppose the number of samples in each pseudo-cluster is $n_i$,  for each pseudo-cluster, the time complexity for creating a complete graph is $\mathcal{O}(n_i^2)$, 
for obtaining the minimum spanning tree from the complete graph using Kruskal's algorithm is $\mathcal{O}(n_i^2\log n_i)$,
and for calculating the shortest path distances between any two points in the pseudo-cluster using Dijkstra's algorithm is $\mathcal{O}(n_i^2\log n_i)$. 
The total time complexity for splitting all existing pseudo-clusters once is $\sum_{i} \mathcal{O}(n_i^2\log n_i)$, where $\sum_{i} n_i = n$. The time complexity for spectral clustering is $\mathcal{O}(m^{3})$. 

Actually, $\sum_{i} \mathcal{O}(n_i^2\log n_i)$
is bounded by $\mathcal{O}(n n_{*}\log n_{*})$, where $n_{*}$ is the maximum of $n_i$, as shown in the following theorem (see the proof). 
Therefore, the overall time complexity of the Algorithm~\ref{alg:algorithm2} is thus $\mathcal{O}(n n_{*}(\log n_{*})+m^3)$.

\begin{theorem}
    Suppose there are $l$ pseudo-clusters. Let $n=\sum_{i=1}^l n_i$, where $n_i\ge 1$ is the number of points in a pseudo-cluster $p_i$. 
    Then $\sum_{i} n_i^2\log n_i = \mathcal{O}(nn_{*} \log {n_{*}})$, where $n_{*}$ is the maximum of $n_i$.
\end{theorem}
\begin{proof}
    Note that $a_1^2\log a_1 +a_2^2\log a_2 \le (a_1+a_2)^2 \log (a_1+a_2)$ and more generally $a_1^2 \log a_1 + \ldots + a_s^2 \log a_s \le (a_1 + \ldots a_s)^2 \log (a_1+\ldots +a_s)$ for any $a_i \ge 1$.

    Therefore, we can group the pseudo-clusters into $C$ groups such that the total number of points in each group is in $[n_{*}/2 , n_{*}]$, with the possible exception of one group whose total number of points is less than $n_{*}/2$.
    Then there are at most $2\lceil n/n_{*}\rceil+1$ groups.
    
    For each group containing the pseudo-clusters $p_{i_1}, \ldots, p_{i_s}$, we have that the corresponding summation terms $n_{i_1}^2\log{n_{i_1}}+\ldots+n_{i_s}^2\log{n_{i_s}} \le (n_{i_1}+\ldots+n_{i_s})^2 \log{(n_{i_1}+\ldots+n_{i_s})}$. Because of the grouping strategy, $n_{i_1}+\ldots+n_{i_s}\le n_*$, the summation of these terms is bounded by $n_*^2 \log n_*$ and $\sum_{i} n_i^2\log n_i \le C n_*^2 \log n_*$.
    Since $C \le 2\lceil n/n_{*}\rceil+1$, $\sum_{i} n_i^2\log n_i \le (2\lceil n/n_{*}\rceil+1) n_*^2 \log{n_*} = \mathcal{O}(nn_{*} \log {n_{*}})$.

\end{proof}

\section{Experiments}
\subsection{Experimental Setup}
The experimental environment is: Windows 10, CPU i5-9300H, GPU NVIDIA GeForce GTX 1050 and 8GB RAM, Python 3.9.

The experiments are conducted on 4 synthetic datasets: Aggregation \cite{DBLP:journals/tkdd/GionisMT07}, Spiral \cite{DBLP:journals/pr/ChangY08}, Jain \cite{DBLP:conf/premi/JainL05}, db2 \cite{DBLP:conf/kdd/EsterKSX96} and 13 UCI datasets \footnote{https://archive.ics.uci.edu/}.
We used three evaluation metrics: ARI~\cite{steinley2004properties}, NMI~\cite{DBLP:conf/sigir/XuLG03}, and ACC~\cite{DBLP:journals/tip/YangXNYZ10} for clustering analysis. 
Since the proposed algorithm is based on local density peaks and pseudo-cluster splitting followed by spectral clustering, we selected the following eight comparison algorithms:
GBSC~\cite{DBLP:journals/tkde/XieKXWG23}, GB-DP~\cite{cheng2023fast}, LDP-MST~\cite{DBLP:journals/tkde/ChengZHWY21}, USPEC~\cite{DBLP:journals/tkde/HuangWWLK20}, spectral clustering (SC)~\cite{DBLP:journals/pami/ShiM00}, DPC~\cite{rodriguez2014clustering}, DEMOS~\cite{DBLP:journals/tkde/GuanLCHC23}, and DPC-DBFN~\cite{DBLP:journals/pr/LotfiMB20}. 

The implementations of GBSC, GB-DP, LDP-MST, USPEC, DEMOS, and DPC-DBFN are from the source code provided by the authors.
The implementation of SC is from the scikit-learn library, and the implementation of DPC is based on the algorithm described in the paper.
For a fair comparison, we performed a search for the optimal hyperparameters required by each algorithm.
Among them, GBSC and GB-DP do not require hyperparameter settings. U-SPEC has a hyperparameter $p$; if the number of data points $n$ is greater than 1000, we set $p$ to 1,000 and to $n$ ($n$ represents the number of samples in the data) otherwise.
For LDP-MST, we followed the experimental setup from the original paper and applied PCA to reduce the dimensionality of datasets with more than 10 dimensions, selecting the best result between 2 and 10 dimensions. 
For SC, we used the library function to construct similarity using k-nearest neighbors. 
For DPC, the suggested average number of points within the cutoff distance should be 1\% to 2\% of the total data. 
We searched this ratio from 1.0 to 2.0 with a step size of 0.1 to find the optimal result. 
For DEMOS, we manually drew rectangular boxes in the decision graph to select cluster center points.The value of $k$ for the number of nearest neighbors is set according to $\sqrt{n}$ ($n$ represents the number of samples in the data) as described in the original paper.
For DPC-DBFN, we searched for the optimal $k$ in the range of 2 to 100. 
For MDMSC, we searched for $k$ in the range of 2 to 50 and for $\beta$ in $\{8, 16\}$, and the threshold $\lambda$ was set to $1.5$. 

All datasets in Table~\ref{tab:real} were normalized, and the results were averaged over 10 runs. The implementation of MDMSC can be found at \url{https://github.com/SWJTU-ML/MDMSC}.

\begin{table}[tb]
\centering
\setlength{\tabcolsep}{1mm}
\small
\begin{tabular}{cccc}
\toprule
Dataset     & \#Instance & \#Attributes & \#Clusters \\ 
\midrule
border                       & 840      & 892        & 3        \\
mfea-fac                     & 2000     & 216        & 10       \\
Kdd9                         & 1280     & 41         & 3        \\
landsatEW                    & 6435     & 36         & 6        \\
balance-scale                & 625      & 4          & 3        \\
pengleukEW                   & 72       & 7070       & 2        \\
Pendigits                    & 10992    & 16         & 10       \\
energy-y2                    & 766      & 8          & 3        \\
optical\_test                & 1797     & 62         & 10       \\
soybean\_test                & 376      & 35         & 18       \\
car                          & 1728     & 6          & 4        \\
semeionEW                    & 1593     & 256        & 10       \\
vote                         & 435      & 16         & 2       \\ 
\bottomrule
\end{tabular}
\caption{Real-world datasets.}\label{tab:real}
\end{table}

\begin{figure}[htbp]
\centering
\includegraphics[width=\columnwidth]{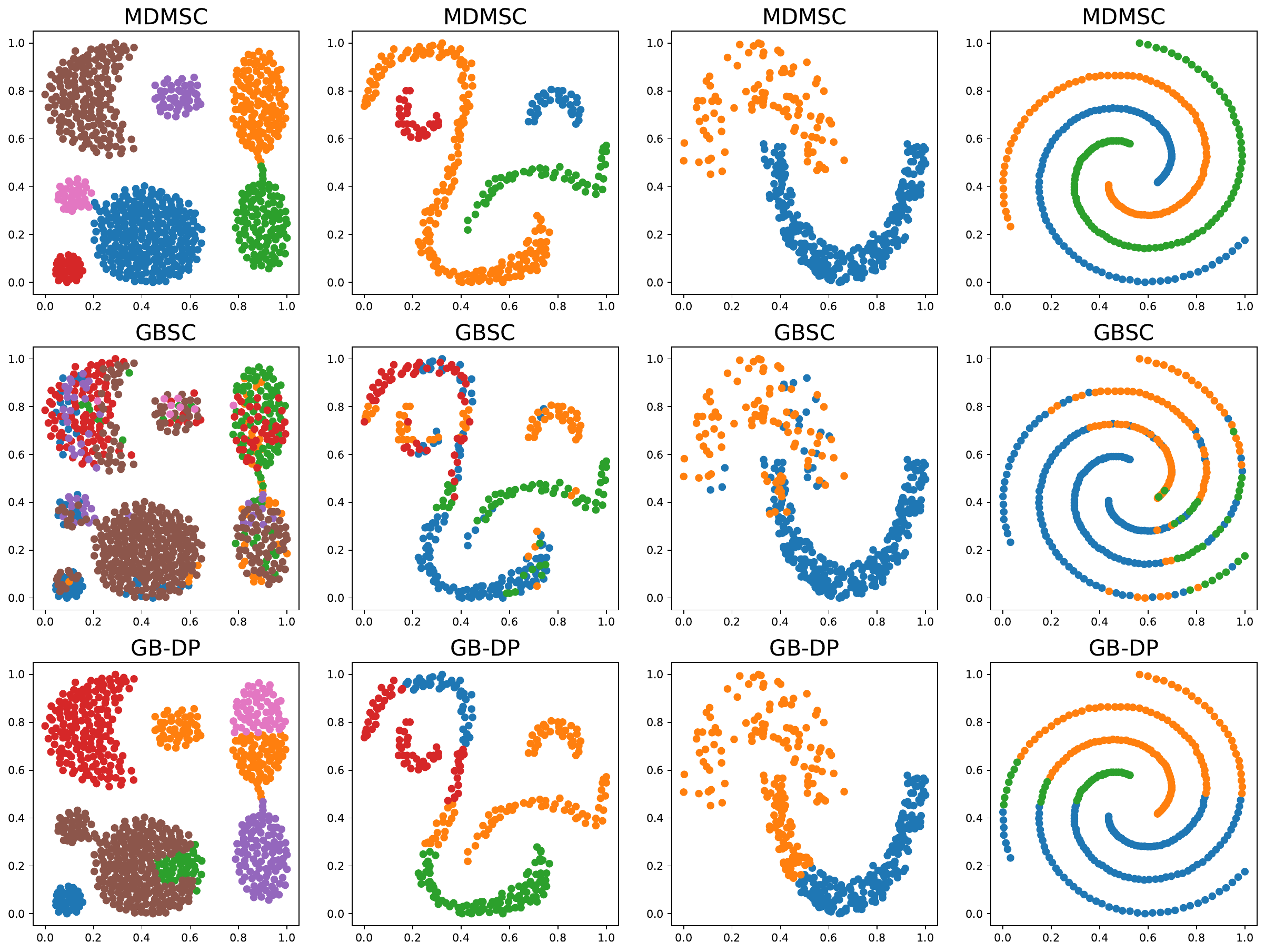}
\caption{Visualizations on synthetic datasets.}\label{fig:synthetic}
\end{figure}

\subsection{Visualizations on Synthetic Datasets}
We present visualizations of several datasets with obvious manifold structures in Fig. \ref{fig:synthetic}, where MDMSC performs well,
while GBSC and GB-DP misclassify points within the same cluster.
For example, GBSC misclassifies many isolated points into incorrect clusters due to the granular-ball splitting process and its simplistic similarity measure. On the other hand, MDMSC, by leveraging local density features and considering manifold curvature, effectively captures complex shapes and manifold structures.

\begin{table*}[tb]
\centering
\setlength{\tabcolsep}{1mm}
\small
\begin{tabular}{ccccccccccc}
\toprule
Dataset              &      & MDMSC           & GBSC  & GB-DP          & LDP-MST        & USPEC          & SC             & DPC   & DEMOS & \multicolumn{1}{l}{DPC-DBFN} \\
\midrule
border               & ARI  & \textbf{19.54} & 0.01  & 16.70          & 1.18           & 17.35          & 17.48          & 8.73  & 14.20 & 2.83                         \\
                     & NMI  & \textbf{18.59} & 0.34  & 15.17 & 1.89           & 16.15          & 15.58          & 10.92 & 14.66 & 6.10                         \\
                     & ACC  & \textbf{56.67} & 44.39 & 49.76          & 45.00          & 53.46          & 51.90          & 49.52 & 52.50 & 45.59                        \\
mfea-fac             & ARI  & \textbf{86.28} & 24.92 & 39.78          & 61.50          & 84.57          & 85.57          & 59.85 & 75.58 & 48.80                        \\
                     & NMI  & \textbf{87.54} & 42.58 & 58.79          & 72.41          & 86.39          & 87.31          & 74.43 & 82.57 & 69.11                        \\
                     & ACC  & \textbf{93.45} & 41.30 & 61.90          & 72.15          & 92.56          & 93.00          & 65.85 & 87.30 & 59.10                        \\
Kdd9                 & ARI  & \textbf{96.53} & 5.67  & 84.93          & 86.08          & 70.71          & 52.16          & 0.05  & NA   & 3.02                         \\
                     & NMI  & \textbf{96.09} & 7.96  & 85.98          & 85.79          & 73.08          & 62.55          & 0.36  & NA   & 13.75                        \\
                     & ACC  & \textbf{97.97} & 41.67 & 94.53          & 95.00             & 86.55          & 73.85          & 38.83 & NA   & 46.56                        \\
landsatEW            & ARI  & 56.45 & 49.17 & 34.48          & 50.01          & \textbf{58.36}          & 47.23          & 46.18 & NA   & 4.68                         \\
                     & NMI  & 63.50          & 56.53 & 47.45          & 59.83          & \textbf{65.16} & 60.42          & 56.06 & NA   & 10.40                        \\
                     & ACC  & 70.16 & 67.61 & 55.63          & 64.21          & 70.50          & 63.66          & \textbf{71.75} & NA   & 30.58                        \\
balance-scale        & ARI  & \textbf{24.60} & 5.25  & 10.31          & 0.91           & 11.10          & 13.57          & 15.07 & NA   & 11.14                        \\
                     & NMI  & \textbf{22.66} & 5.44  & 8.91           & 6.63           & 9.33           & 12.15          & 11.73 & NA   & 10.42                        \\
                     & ACC  & \textbf{60.16} & 51.70 & 52.48          & 49.92          & 51.94          & 54.24          & 54.40 & NA   & 54.40                        \\
pengleukEW           & ARI  & \textbf{35.85} & 2.96  & 23.79          & 32.90          & 29.75          & 26.63          & 10.52 & 26.00 & 15.95                        \\
                     & NMI  & \textbf{25.07} & 5.95  & 17.13          & 24.06          & 21.96          & 18.95          & 14.39 & 16.98 & 10.73                        \\
                     & ACC  & \textbf{80.56} & 58.61 & 75.00          & 79.16          & 77.78          & 76.39          & 70.83 & 76.38 & 70.83                        \\
Pendigits            & ARI  & \textbf{77.81} & 42.62 & 51.91          & 70.11          & 71.05          & 76.24          & 63.50 & 70.53 & 49.91                        \\
                     & NMI  & \textbf{84.85} & 59.72 & 67.51          & 81.66          & 81.52          & 83.73          & 75.41 & 80.75 & 67.67                        \\
                     & ACC  & \textbf{88.08} & 57.17 & 63.05          & 78.02          & 82.98          & 87.17          & 75.63 & 84.18 & 64.55                        \\
energy-y2            & ARI  & \textbf{70.65} & 4.57  & 69.17          & 70.58          & 35.24          & 70.58          & 55.17 & 45.58 & 65.29                        \\
                     & NMI  & 66.44 & 5.55  & \textbf{68.97}          & 66.40          & 47.02          & 66.40          & 62.13 & 58.11 & 66.67                        \\
                     & ACC  & \textbf{80.86} & 51.07 & 74.09          & 80.73          & 55.95          & 80.73          & 65.76 & 49.86 & 71.61                        \\
optical\_test        & ARI  & \textbf{84.08} & 28.70 & 32.52          & 56.37          & 80.84          & 81.47          & 72.25 & 82.27 & -0.06                        \\
\multicolumn{1}{l}{} & NMI  & \textbf{90.13} & 51.58 & 59.44          & 75.04          & 88.06          & 89.91          & 83.41 & 87.90 & 0.83                         \\
                     & ACC  & \textbf{89.43} & 52.08 & 55.15          & 65.77          & 86.69          & 87.82          & 78.69 & 89.14 & 10.51                        \\
soybean\_test        & ARI  & 43.04          & 29.26 & 31.26          & 39.89          & 39.11          & \textbf{46.22} & 32.72 & 31.30 & 35.13                        \\
                     & NMI  & \textbf{75.08} & 59.03 & 65.55          & 67.21          & 74.81          & 69.46          & 64.65 & 59.82 & 58.29                        \\
\multicolumn{1}{l}{} & ACC  & \textbf{64.10} & 44.55 & 42.82          & 53.98          & 60.61          & 56.41          & 46.01 & 44.95 & 50.80                        \\
car                  & ARI  & \textbf{40.44} & -3.08 & 4.90           & 11.40          & 17.19          & 17.79          & 15.90 & 9.16  & 14.46                        \\
                     & NMI  & 28.51          & 5.28  & 11.13          & \textbf{32.40} & 24.57          & 25.04          & 30.47 & 15.73 & 10.36                        \\
                     & ACC  & \textbf{69.44} & 57.25 & 35.94          & 55.55          & 47.39          & 47.96          & 51.56 & 46.81 & 61.57                        \\
semeionEW            & ARI  & \textbf{52.62} & 0.53  & 23.28          & 26.79          & 48.34          & 44.81          & 19.63 & 33.19 & 26.03                        \\
                     & NMI  & \textbf{67.45}          & 5.85  & 42.15          & 52.15          & 64.66 & 64.29          & 35.20 & 55.53 & 42.59                        \\
                     & ACC  & \textbf{68.30} & 13.21 & 41.68          & 47.70          & 62.82          & 58.59          & 36.97 & 45.26 & 39.30                        \\
vote                 & ARI  & \textbf{57.10} & -0.83 & 53.67          & 20.57          & 0.37           & 55.72          & 53.00 & NA   & 45.07                        \\
                     & NMI  & \textbf{49.93} & 8.32  & 45.98          & 24.44          & 0.31           & 48.40          & 45.95 & NA   & 33.92                        \\
                     & ACC  & \textbf{87.82} & 55.08 & 86.67          & 72.87          & 61.61          & 87.36          & 86.44 & NA   & 83.67                        \\
                    \bottomrule
\end{tabular}
\caption{Results on real-world datasets (\%).}\label{tab:realres}
\end{table*}

\begin{table*}[tb]
\centering
\setlength{\tabcolsep}{1mm}  
\small
\begin{tabular}{ccccccccccc}
\toprule
DataSet     &               & Ours                 & GBSC  & GB-DP          & LDP-MST        & USPEC          & SC             & DPC            & DEMOS & \multicolumn{1}{l}{DPC-DBFN} \\ \hline
                  &Par.& $k$\textbackslash$\beta$  &  & & dimension     & $p$                & $k$   & dc    & $k$     & $k$     \\ \hline           
border            &    & 24 & \textbackslash  & \textbackslash  & 7  & 840   & 13  & 1.0   & 29    & 2        \\
mfea-fac          &    & 6  & \textbackslash  & \textbackslash  & 8  & 1000  & 9   & 1.4   & 45    & 18        \\
Kdd9              &    & 21 & \textbackslash  & \textbackslash  & 2  & 1000  & 14  & 1.0   & NA    & 100       \\
landsatEW         &    & 43 & \textbackslash  & \textbackslash  & 6  & 1000  & 5   & 1.4   & NA    & 99         \\
balance-scale     &    & 44 & \textbackslash  & \textbackslash  & 4  & 625   & 11  & 1.7   & NA    & 84         \\
pengleukEW        &    & 3  & \textbackslash  & \textbackslash  & 10 & 72    & 11  & 1.0   & 8     & 43          \\
Pendigits         &    & 23 & \textbackslash  & \textbackslash  & 7  & 1000  & 24  & 1.0   & 105   & 91           \\
energy-y2         &    & 20 & \textbackslash  & \textbackslash  & 8  & 766   & 9   & 1.0   & 28    & 15           \\
optical\_test     &    & 12 & \textbackslash  & \textbackslash  & 9  & 1000  & 10  & 1.9   & 42    & 3             \\
soybean\_test     &    & 4  & \textbackslash  & \textbackslash  & 6  & 376   & 49  & 1.5   & 19    & 25            \\
car              &     & 8  & \textbackslash  & \textbackslash  & 6  & 1000  & 31  & 1.1   & 42    & 31            \\
semeionEW        &     & 3  & \textbackslash  & \textbackslash  & 6  & 1000  & 5   & 1.0   & 40    & 5           \\
vote             &     & 46  & \textbackslash & \textbackslash  & 6  & 435   & 33  & 2.0   & NA    & 97            \\ 
                    \bottomrule
\end{tabular}
\caption{Hyperparameter values of the results on real-world datasets.}\label{tab:pars}
\end{table*}
\subsection{Results on Real-World Datasets}
Table~\ref{tab:realres} shows the results (bold means best) of our algorithm and the other eight comparison algorithms on real-world datasets. 
Table~\ref{tab:pars} shows their corresponding hyperparameter values. The "\textbackslash" indicates that the corresponding algorithm does not require hyperparameters, while "NA" indicates that the algorithm fails to produce results. MDMSC achieves the best performance on most datasets. For example, on the high-dimensional dataset pengleukEW, MDMSC significantly outperforms most comparison algorithms. We attribute this to the consideration of manifold curvature, allowing it to effectively handle complex high-dimensional datasets. 
Additionally, on the large-scale dataset Pendigits, our algorithm also shows superior performance, demonstrating its capability to effectively handle large-scale datasets by partitioning pseudo-clusters. 
Although our algorithm does not always achieve the best results on a few datasets like landsatEW and soybean\_test, its performance is very close to the optimal ones. This indicates that our algorithm possesses strong adaptability and effectiveness across a wide range of application scenarios.

Besides, We first performed a Friedman test on the ACC matric of MDMSC and the comparison algorithms. After confirming significant differences between the algorithms, we conducted a Nemenyi post-hoc test to compute the statistical significance of the differences between MDMSC and the comparison algorithms.
We obtained $Friedman Statistic=49.37$ with $P=2.10e-08<0.05$, confirming significant differences between the algorithms. In the Nemenyi post-hoc test, the critical difference threshold $CD=2.10$. Fig. \ref{CD} illustrates the mean rank of all algorithms, where a rank difference greater than CD indicates a significant difference between the two algorithms.

\begin{figure}[htbp]
\centering
\includegraphics[width=\columnwidth]{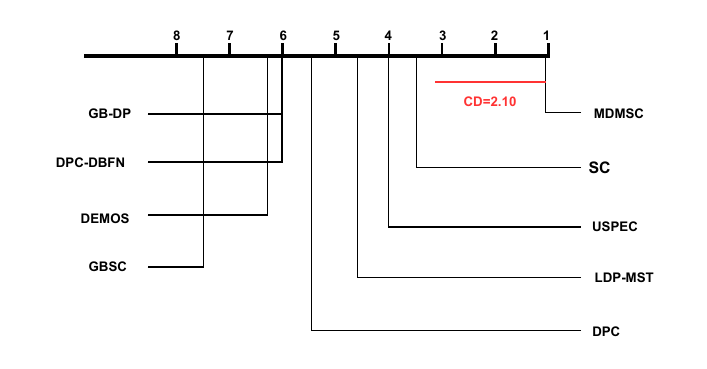} 
\caption{Mean rank of all algorithms on the ACC metric.}
\label{CD}
\end{figure}

\subsection{Ablation Study}
To demonstrate the effectiveness of each component of our algorithm, we conducted ablation studies on 7 real-world datasets. The best results are presented in Table~\ref{tab:ab} by tuning hyperparameters. 
The columns $a$, $b$, and $c$ represent the following different settings, respectively:
\begin{enumerate}
\item[$a$:] Directly uses the whole dataset as the initial pseudo-cluster, without using Algorithm~\ref{alg:algorithm}.
\item[$b$:]  
Does not further split the pseudo-clusters. 
\item[$c$:]  Only uses compactness to split the pseudo-clusters.
\end{enumerate}

The results indicate that the performance of the algorithm decreases when a specific component is removed. Setting $a$ demonstrates that extracting local density features helps characterize the data distribution. Settings $b$ and $c$ show that considering manifold curvature is beneficial for the quality of pseudo-clusters. 
\begin{table}[tb]
\centering
\setlength{\tabcolsep}{1mm}
\small
\begin{tabular}{ccccc}
\toprule
Dataset       & MDMSC           & $a$     & $b$     & $c$              \\ 
\midrule
mfea-fac      & \textbf{86.28} & 71.28 & 71.72 & 84.40          \\
Kdd9          & \textbf{96.53} & 84.93 & 84.93 & \textbf{96.53} \\
balance-scale & \textbf{24.60} & 14.40 & 19.61 & 22.46          \\
pengleukEW    & \textbf{35.85} & 29.61 & -5.84 & 23.79          \\
optical\_test & \textbf{84.08} & 63.56 & 78.11 & 81.53          \\
semeionEW     & \textbf{52.62} & 34.90 & 46.57 & 48.12          \\
vote          & \textbf{57.10} & 54.34 & 53.00 & 53.00  
\\ \bottomrule
\end{tabular}
\caption{Ablation study results (\%).}\label{tab:ab}
\end{table}

\subsection{Robustness Analysis}
Fig. \ref{fig11} demonstrates the ARI variations of MDMSC under different hyperparameter settings on multiple datasets, where $k$ was set in $[3,50]$, $\lambda$ in $[1.0, 3.0]$ with a step size of 0.2, and $\beta$ in $[8,16]$. 
From the results in Fig. \ref{fig11}, it is evident that while the algorithm exhibits some fluctuations with respect to $k$ (larger fluctuations with smaller $k$ and stablized as $k$ increases), its performance remains relatively stable when varying $\lambda$ and $\beta$. This suggests that MDMSC has a certain degree of robustness against changes in these parameters.

\begin{figure}[htbp]
\centering
\includegraphics[width=\columnwidth]{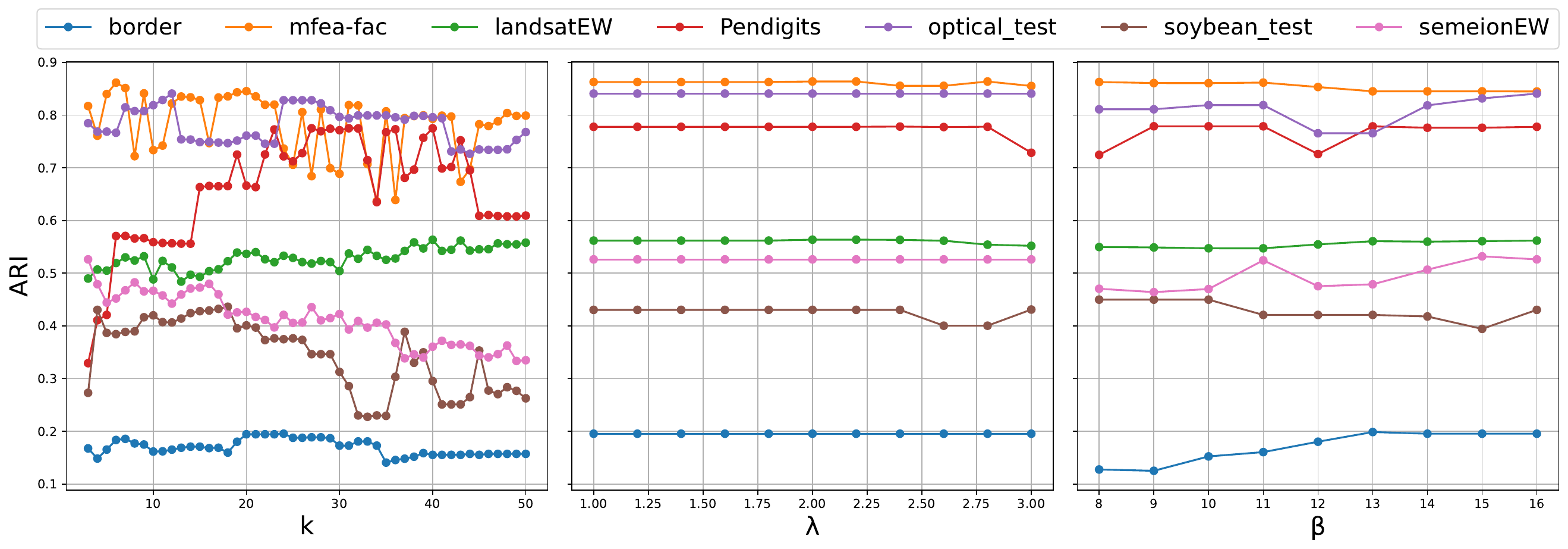} 
\caption{Impacts of $k$ on clustering performance.}
\label{fig11}
\end{figure}

\subsection{Running Time}
Fig. \ref{fig12} compares the running time and ACC of our algorithm with GBSC.
The proposed MDMSC has shorter run time than GBSC on 7 out of 13 datasets and similar run time on most of the other datasets, and on all of these datasets, MDMSC has higher ACC than GBSC. 
The reason for the slower cases is that we make use of more complex micro-clusters to represent data and the splitting rule is also more sophisticated, and it is actually worthy in most of the cases.

\begin{figure}[tb]
\centering
\includegraphics[width=0.9\columnwidth]{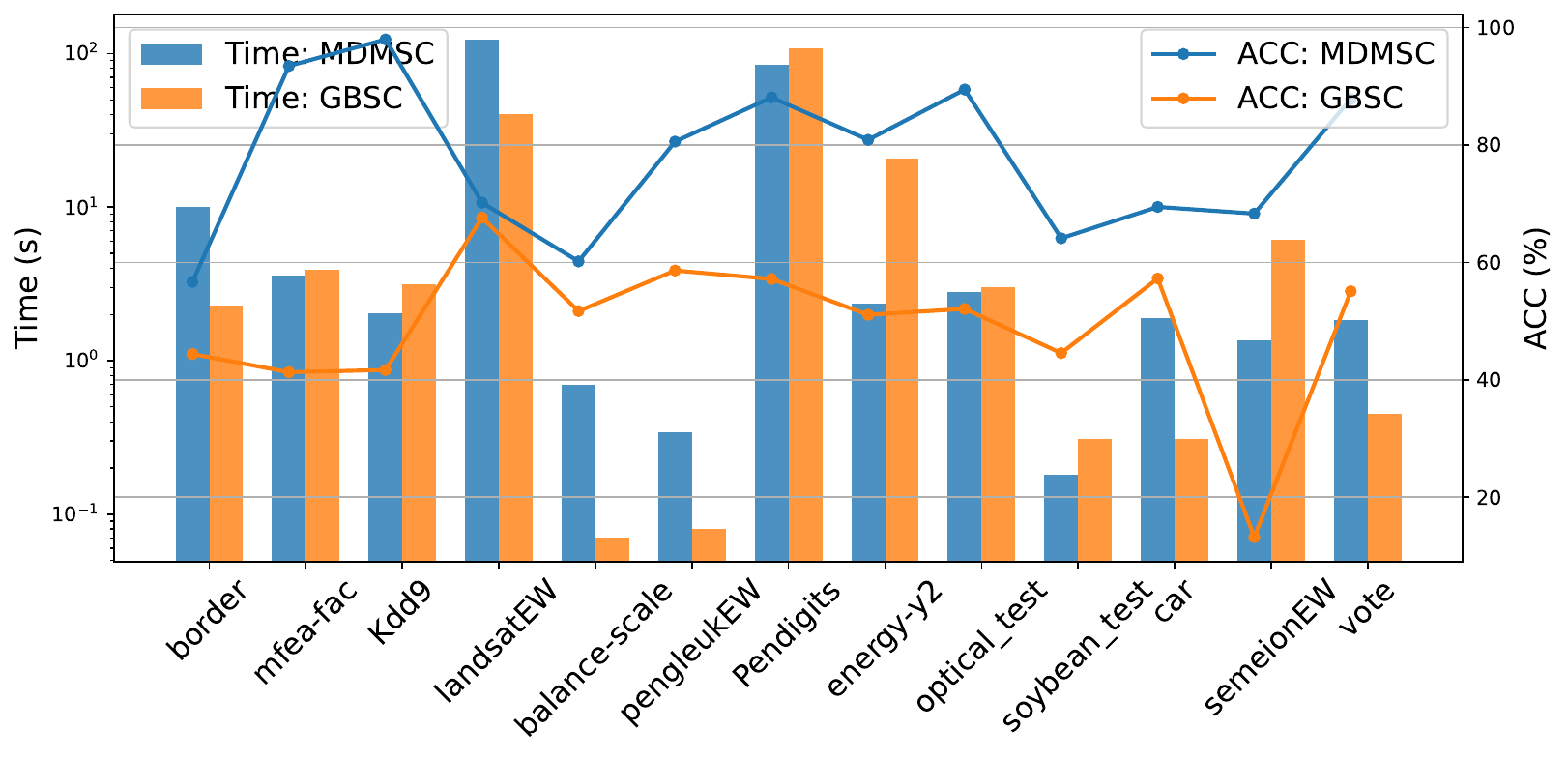} 
\caption{Time and ACC comparison.}
\label{fig12}
\end{figure}

\section{Conclusion}
The paper proposes a new approach to represent data in a coarse granularity that can accelerate and improve the performance of spectral clustering.
Specifically, the approach can discover micro-clusters based on local density distributions of data, and introduces the concept of manifold curvature of micro-clusters to help split them into more convex ones.
In this way, the approach provides better representation of data and simplifies the characterization of the similarities, resulting in better performance in subsequent spectral clustering. 
 Evaluations on 4 synthetic datasets and 13 real-world datasets, against relevant state-of-the-art algorithms, show that  
 the proposed algorithm performs best on most datasets. The ablation experiments also demonstrate the effectiveness of the proposed components.

There is still room for improvement and optimization. 
Future work will attempt to introduce more efficient computational techniques, such as parallel and distributed computing, 
to further enhance the computational speed of the algorithm, 
especially when handling ultra-large-scale datasets. 
Additionally, one can consider developing adaptive hyperparameter optimization methods for this approach, and can also consider to combine it with other clustering methods.

\section{Acknowledgments}
We would like to thank the anonymous reviewers for their invaluable help to improve the paper. This work was supported by the National Natural Science Foundation of
China [Nos. 61806170 and 62276218], the Fundamental Research Funds for the Central Universities [Nos. 2682022ZTPY082 and 2682023ZTPY027], and the Open Foundation of Key Laboratory of Cyberspace Security, Ministry of Education of China and Henan Key Laboratory of Cyberspace Situation Awareness [No.KLCS20240105].
\bibliography{aaai25}
\end{document}